\documentclass[letterpaper,11pt,lualatex]{article}
\usepackage[margin=1in, centering]{geometry}
\usepackage[utf8]{luainputenc}
\usepackage[OT1]{fontenc}
\usepackage{graphicx,xcolor}
\usepackage[cmex10]{amsmath}
\usepackage{amsthm,amssymb,amsfonts,mathrsfs,bm,mathtools,cases}
\usepackage{siunitx}
\usepackage{subcaption}
\usepackage{setspace}
\usepackage[inline]{enumitem}
\usepackage{etoolbox}
\usepackage[linesnumbered,vlined,ruled,commentsnumbered]{algorithm2e}
\usepackage{thmtools}
\usepackage[unq]{unique}
\usepackage[capitalize]{cleveref}
\usepackage{cite}

\DeclareGraphicsExtensions{.pdf}
\declaretheoremstyle[%
headfont=\normalfont\bfseries,
notefont=\mdseries,
notebraces={(}{)},
bodyfont=\normalfont,
postheadspace=1ex,
qed=\qedsymbol%
]{mystyle}

\declaretheorem[style=mystyle,
                name=Theorem,
                refname={theorem,theorems},
                Refname={Theorem,Theorems}
]{thm}

\declaretheorem[style=mystyle,
                name=Assumption,
                refname={assumption,assumptions},
                Refname={Assumption,Assumptions}
]{assumption}

\newlist{thmlist}{enumerate}{1}
\setlist[thmlist]{label=\textbf{(\roman{*})},
  ref=\thethm(\roman{*}),noitemsep}

\newlist{lemlist}{enumerate}{1}
\setlist[lemlist]{label=\textbf{(\roman{*})},
  ref=\thelemma(\roman{*}),noitemsep}

\newlist{exlist}{enumerate}{1}
\setlist[exlist]{label=\textbf{(\roman{*})},
  ref=\theexample(\roman{*}),noitemsep}

\newlist{factlist}{enumerate}{1}
\setlist[factlist]{label=\textbf{(\roman{*})},
  ref=\thefact(\roman{*}),noitemsep}

\newlist{proplist}{enumerate}{1}
\setlist[proplist]{label=\textbf{(\roman{*})},
  ref=\theprop(\roman{*}),noitemsep}

\newlist{asslist}{enumerate}{1}
\setlist[asslist]{label=\textbf{(\roman{*})},
  ref=\theassumption(\roman{*}),noitemsep}

\newlist{deflist}{enumerate}{1}
\setlist[deflist]{label=\textbf{(\roman{*})},
  ref=\thedefinition(\roman{*}),noitemsep}

\newlist{algolist}{enumerate}{1}
\setlist[algolist]{label=\textbf{(\roman{*})},
  ref=\thealgo(\roman{*}),noitemsep}

\newlist{claimlist}{enumerate}{1}     
\setlist[claimlist]{label=\textbf{(\roman{*})},
  ref=\theclaim(\roman{*}),noitemsep}

\newlist{applist}{enumerate}{1}
\setlist[applist]{label=\textbf{(\roman{*})},
  ref=\thesection(\roman{*}),noitemsep}

\crefname{thm}{Theorem}{Theorems}
\crefname{prop}{Proposition}{Propositions}
\crefname{assumption}{Assumption}{Assumptions}
\crefname{lemma}{Lemma}{Lemmata}
\crefname{definition}{Definition}{Definitions}
\crefname{example}{Example}{Examples}
\crefname{algo}{Algorithm}{Algorithms}
\crefname{fact}{Fact}{Facts}
\crefname{claim}{Claim}{Claims}
\crefname{appendix}{Appendix}{Appendices}
\crefname{coroll}{Corollary}{Corollaries}
\crefname{figure}{Figure}{Figures}
\crefname{section}{Section}{Sections}

\crefname{thmlisti}{Theorem}{Theorems}
\crefname{lemlisti}{Lemma}{Lemmata}
\crefname{proplisti}{Proposition}{Propositions}
\crefname{asslisti}{Assumption}{Assumptions}
\crefname{deflisti}{Definition}{Definitions}
\crefname{exlisti}{Example}{Examples}
\crefname{algolisti}{Algorithm}{Algorithms}
\crefname{factlisti}{Fact}{Facts}
\crefname{claimlisti}{Claim}{Claims}
\crefname{applisti}{Appendix}{Appendices}

\newcommand{\Real}{\mathbb{R}}

\newcommand{\RealPP}{\mathbb{R}_{>0}}

\newcommand\given{{\mathbin{}\mid\mathbin{}}}
\newcommand\vect[1]{\mathbf{#1}}

\newcommand\limas{\xrightarrow{\text{a.s.}}}

\providecommand\given{} 
\newcommand\SetSymbol[1][]{
  \nonscript\,#1\vert \allowbreak \nonscript\,\mathopen{}}
\DeclarePairedDelimiterX\Set[1]{\lbrace}{\rbrace}%
{\renewcommand\given{\SetSymbol[\delimsize]} #1 }
\DeclarePairedDelimiterX\innerp[2]{\langle}{\rangle}{#1
  \mathop{}\delimsize\vert\mathop{} #2}
\DeclarePairedDelimiterX\norm[1]\lVert\rVert{\ifblank{#1}{\:\cdot\:}{#1}}

\DeclareMathOperator{\domain}{Dom}

\DeclareMathOperator{\trace}{Trace}

\DeclareMathOperator{\Expect}{\mathbb{E}}

\DeclareMathOperator{\Prox}{Prox}

\DeclareMathOperator*{\Argmin}{arg\,min}

\makeatletter

\newcommand*{\ie}{%
  \@ifnextchar{,}%
  {\textit{i.e.}}%
  {\textit{i.e.,}\@\xspace}%
}
\newcommand*{\eg}{%
  \@ifnextchar{,}%
  {\textit{e.g.}}%
  {\textit{e.g.,}\@\xspace}%
}
\newcommand*{\etc}{%
  \@ifnextchar{.}%
  {\textit{etc}}%
  {\textit{etc.}\@\xspace}%
}
\newcommand*{\etal}{%
  \@ifnextchar{.}%
  {\textit{et al}}%
  {\textit{et al.}\@\xspace}%
}
\newcommand*{\cf}{%
  \@ifnextchar{.}%
  {\textit{cf}}%
  {\textit{cf.}\@\xspace}%
}
\newcommand*{\aka}{%
  \@ifnextchar{,}%
  {\textit{a.k.a.}}%
  {\textit{a.k.a.}\@\xspace}%
}
\makeatother

\title{Robust Hierarchical-Optimization RLS\\
Against Sparse Outliers}

\author{Konstantinos~Slavakis and Sinjini~Banerjee%
  \thanks{K.~Slavakis (contact author) and S.~Banerjee are with the Dept.\ of
    Electrical Engineering, University at Buffalo, The State University of
    New~York, Buffalo, NY 14260, USA (e-mail:~\mbox{kslavaki@buffalo.edu}).}}

\date{}

\begin{document}
\maketitle

\begin{abstract}
  This paper fortifies the recently introduced hierarchical-optimization
  recursive least squares (HO-RLS) against outliers which contaminate
  infrequently linear-regression models. Outliers are modeled as nuisance
  variables and are estimated together with the linear filter/system variables
  via a sparsity-inducing (non-)convexly regularized least-squares task. The
  proposed outlier-robust HO-RLS builds on steepest-descent directions with a
  constant step size (learning rate), needs no matrix inversion (lemma),
  accommodates colored nominal noise of known correlation matrix, exhibits small
  computational footprint, and offers theoretical guarantees, in a probabilistic
  sense, for the convergence of the system estimates to the solutions of a
  hierarchical-optimization problem: Minimize a convex loss, which models
  a-priori knowledge about the unknown system, over the minimizers of the
  classical ensemble LS loss. Extensive numerical tests on synthetically
  generated data in both stationary and non-stationary scenarios showcase
  notable improvements of the proposed scheme over state-of-the-art techniques.
\end{abstract}


\section{Introduction}\label{Sec:Intro}

The recursive least squares (RLS) has been a pivotal method in solving LS
problems in adaptive filtering and system identification~\cite{SayedBook}, with
a reach that extends also into contemporary learning tasks, such as solving
large-scale LS problems in online learning, \eg, \cite{ELM:05}. Nevertheless,
the performance of RLS (LS estimators in general) deteriorates in the presence
of outliers, \ie, data or noise not adhering to a nominal data-generation
model~\cite{Rousseeuw:book.87}. This work focuses on outliers that contaminate
infrequently data models, \eg, impulse noise~\cite{Fano:Impulse.noise:77,
  Lin:impulse.noise:13, Nikolova:04}.

Methods that strengthen RLS against outliers have been reported
in~\cite{Petrus:99, RLM:00, Chan.RLM:04, ReyVega.RobustRLS:09, Bhotto:11,
  Farahmand.robustRLS.12, Xiao.robustRLS.15}. Propelled by robust-regression
arguments~\cite{Rousseeuw:book.87}, studies \cite{Petrus:99, RLM:00,
  Chan.RLM:04} utilize M-estimate losses instead of typical LS ones to penalize
system-output errors. The \textit{non-}recursive algorithm of \cite{Petrus:99}
employs Huber's loss, while the recursive schemes of \cite{RLM:00} and
\cite{Chan.RLM:04} build on Hampel's three-part redescending
objective~\cite{Rousseeuw:book.87} and a modified Huber's loss,
respectively. Numerical tests show that \cite{Petrus:99} outperforms median
filters~\cite{Gallagher:Median:81}, a classical solution to mitigate impulse
noise, while \cite{RLM:00, Chan.RLM:04} appear to be more effective than
order-statistics techniques~\cite{Settineri:OSFKF:96}. Notwithstanding, noise
and error-filtering statistics need to be known a-priori to set the parameters
of Hampel's objective in~\cite{RLM:00}, while the solutions of~\cite{RLM:00,
  Chan.RLM:04} to the M-estimate normal equations are built on the assumption
that filter's estimates do not change significantly for certain amounts of
time~\cite{ReyVega.RobustRLS:09}. Studies~\cite{ReyVega.RobustRLS:09, Bhotto:11}
update filter's estimates by minimizing weighted LS error costs subject to
``ball'' constraints to prevent from large perturbations which may be inflicted
by outliers: A Euclidean-ball constraint onto filter estimates is advocated
by~\cite{ReyVega.RobustRLS:09}, while an $\ell_1$-ball constraint onto the RLS
gain vectors is utilized in~\cite{Bhotto:11}. Numerical tests demonstrate the
improved performance of \cite{ReyVega.RobustRLS:09, Bhotto:11}
over~\cite{RLM:00, Chan.RLM:04}.

Outliers are modeled as nuisance variables and are jointly estimated with the
filter's coefficients in~\cite{Farahmand.robustRLS.12, Xiao.robustRLS.15}. To
address identifiability issues in the estimation task due to the ever-growing
number of the outlier unknowns with recursions (time instances), filter and
outlier vectors are updated per recursion via the minimization of an LS data-fit
loss plus a surrogate of the $\ell_0$-norm of the outlier vector to exploit
sparsity: The $\ell_1$-norm is utilized in \cite{Farahmand.robustRLS.12}, while
non-convex surrogates in~\cite{Xiao.robustRLS.15}. The composite minimization
task is solved in an alternating fashion: First, the outlier vector is updated,
and then the classical RLS is used to update the filter's estimates. Methods
\cite{Farahmand.robustRLS.12, Xiao.robustRLS.15} accommodate also colored
nominal noise of known correlation matrix to avoid any prewhitening that would
spread the outliers in the nominal data, adding further complication to the
challenge of outlier removal. Numerical tests show the improved performance of
\cite{Farahmand.robustRLS.12, Xiao.robustRLS.15}
over~\cite{ReyVega.RobustRLS:09, Bhotto:11}.

This short paper follows the path of~\cite{Farahmand.robustRLS.12,
  Xiao.robustRLS.15} to model outliers as nuisance variables, but employs the
recently introduced hierarchical(-optimization) recursive least squares
(HO-RLS)~\cite{SFMHSDM} to update filter coefficients instead of the classical
RLS. Unlike RLS, which propels all of~\cite{Petrus:99, RLM:00, Chan.RLM:04,
  ReyVega.RobustRLS:09, Bhotto:11, Farahmand.robustRLS.12, Xiao.robustRLS.15},
HO-RLS provides a way to quantify side information about the system since it
solves a hierarchical-optimization problem: Minimize a convex loss, which models
the available side information, over the minimizers of the classical ensemble LS
data-fit loss. The proposed outlier-robust HO-RLS builds on steepest-descent
directions with a constant step size (learning rate), needs \textit{no}\/ matrix
inversion (lemma), exhibits similar computational complexity with the
implementations in~\cite{Farahmand.robustRLS.12, Xiao.robustRLS.15},
accommodates colored noise of known covariance matrix without any prewhitening,
and offers theoretical guarantees, in a probabilistic sense, for the convergence
of the filter/system estimates to the solution of the aforementioned
hierarchical-optimization task. Extensive numerical tests on synthetically
generated data, in both stationary and non-stationary scenarios, showcase
notable improvements of the proposed scheme over the
state-of-the-art~\cite{Farahmand.robustRLS.12, Xiao.robustRLS.15}.

\section{The Problem  and State-Of-The-Art Solutions}\label{Sec:Problem.PriorArt}

With the positive integer $n$ denoting both discrete time and recursion index,
the following data model is considered:
\begin{align}
  \vect{y}_n = \vect{F}_* \vect{x}_n + \vect{o}_n + \vect{v}_n
  \,, \label{data.model}
\end{align}
where the $L\times P$ matrix $\vect{F}_*$ is the wanted filter/system, the
$L\times 1$ vector $\vect{y}_n$ collects the output data, the $P\times 1$ vector
$\vect{x}_n$ gathers the input ones, vector $\vect{o}_n$ models the outlier
data, $\vect{v}_n$ stands for zero-mean (colored) noise with correlation matrix
$\vect{R}_{vv} \coloneqq \Expect\{\vect{v}_n \vect{v}_n^{\intercal}\}$ that is
assumed to stay constant $\forall n$, $\Expect\{\cdot\}$ is the expectation
operator with respect to (w.r.t.) a probability space~\cite{Williams.book.91},
and $\intercal$ denotes vector/matrix transposition. The
multiple-input-multiple-output model \eqref{data.model} is chosen here to offer
a general model that is able to capture data-generation mechanisms in numerous
modern application domains, \eg, \cite{Bereyhi:SPAWC:19,
  Benesty:Acoustic.MIMO:04} and \cite[p.~647]{SayedBook}. To save space,
$\vect{y}_n$, $\vect{x}_n$, $\vect{o}_n$ and $\vect{v}_n$ represent both random
variables (RVs) and their realizations. Furthermore, equality in
\eqref{data.model} is assumed to hold almost surely (a.s.) w.r.t.\ the
underlying probability space. Since outliers $(\vect{o}_n)_{n\geq 1}$ are
assumed to appear infrequently in \eqref{data.model}, vector $\vect{o}_n$ can be
considered to be sparse, \ie, most of its entries are zero $\forall n$. The
input-output data become available to the user sequentially: The pair
$(\vect{x}_n, \vect{y}_n)$ is revealed to the user at time $n$. The problem
under consideration is to devise an iterative algorithm to estimate/learn
$\vect{F}_*$ from the available $(\vect{x}_{\nu}, \vect{y}_{\nu})_{\nu = 1}^n$,
$\forall n$, and $\vect{R}_{vv}$.

For every $n$, the typical LS estimator
\begin{align}
  ( \hat{\vect{F}}_{n+1}, \{\hat{\vect{o}}_{\nu, n}\}_{\nu=1}^n ) 
  \in \Argmin_{ (\bm{F}, \{\bm{o}_{\nu}\}_{\nu=1}^n) } \sum_{\nu=1}^n
  \norm{ \vect{y}_{\nu} - \bm{o}_{\nu} -
  \bm{F}\vect{x}_{\nu}}_{\vect{R}_{vv}^{-1}}^2\,, \label{standard.est.problem}
\end{align}
does not offer significant help, since the trivial
$\hat{\vect{F}}_{n+1} \coloneqq \vect{0}$ and
$\{\hat{\vect{o}}_{\nu, n} \coloneqq \vect{y}_{\nu}\}_{\nu=1}^n$ qualify as
solutions to \eqref{standard.est.problem}. Motivated by classical arguments,
\eg, \cite[\S29.6]{SayedBook}, the weighted norm
$\norm{\vect{a}}_{\vect{R}_{vv}^{-1}}^2 \coloneqq \vect{a}^{\intercal}
\vect{R}_{vv}^{-1} \vect{a}$ is introduced in \eqref{standard.est.problem} to
handle entries of the error vector unequally via the available
$\vect{R}_{vv}^{-1}$. To avoid trivial solutions, a popular way is to form a
regularized LS estimation task
\begin{align}
  ( \hat{\vect{F}}_{n+1}, \{\hat{\vect{o}}_{\nu, n}\}_{\nu=1}^n ) 
    \in \Argmin_{ (\bm{F}, \{\bm{o}_{\nu}\}_{\nu=1}^n) } \sum_{\nu=1}^n
  \left[ \tfrac{1}{2} \norm{ \vect{y}_{\nu} - \bm{F}\vect{x}_{\nu} -
  \bm{o}_{\nu}}_{\vect{R}_{vv}^{-1}}^2 + \lambda_{\nu} \rho(\bm{o}_{\nu})
    \right] \label{Offline.task} \,,
\end{align}
where $\rho(\cdot)$ is a sparsity-inducing loss with user-defined weights
$\lambda_{\nu} > 0$. However, this path seems to be impractical since the
$LP + nL$ number of unknown variables at time $n$ raises insurmountable
computational obstacles when solving \eqref{Offline.task} at large time
instances $n$.

To address this ``curse of dimensionality,'' studies
\cite{Farahmand.robustRLS.12, Xiao.robustRLS.15} adopt the following recursions
for all $n \geq n_0+1$,
\begin{subequations}\label{ORRLS}
  \begin{align}
    & \hat{\vect{o}}_n \in \Argmin_{\bm{o}} \tfrac{1}{2} \norm{ \vect{y}_n -
      \hat{\vect{F}}_n \vect{x}_n - \bm{o}}_{\vect{R}_{vv}^{-1}}^2 +
      \lambda_n \rho(\bm{o})\,, \label{Prior.Art.o} \\
    & \hat{\vect{F}}_{n+1} \coloneqq \text{RLS} \left( \hat{\vect{F}}_n,
      (\vect{x}_n, \vect{y}_n - \hat{\vect{o}}_n) \right)\,, \label{Prior.Art.F}
  \end{align}
\end{subequations}
where $n_0$ is a user-defined time instance,
$\text{RLS}(\hat{\vect{F}}_n, (\vect{x}_n, \vect{y}_n - \hat{\vect{o}}_n))$
denotes the classical RLS updates, with the newly introduced data pair at time
$n$ being $(\vect{x}_n, \vect{y}_n - \hat{\vect{o}}_n)$, and with
$\vect{R}_{vv}^{-1}$ incorporated in the RLS formulae. A warm start is achieved
by solving the offline task \eqref{Offline.task} for
$( \hat{\vect{F}}_{n_0}, \{\hat{\vect{o}}_{\nu}\}_{\nu=1}^{n_0-1} )$, where
$n_0-1$ is used in the place of $n$. Loss $\rho(\cdot)$ takes the form of the
$\ell_1$-norm in \cite{Farahmand.robustRLS.12}, rendering \eqref{Offline.task}
and \eqref{Prior.Art.o} typical LASSO tasks~\cite{Hastie:book:09}, while
non-convex $\rho(\cdot)$ are promoted in \cite{Xiao.robustRLS.15} and non-convex
optimization solvers, \eg, \cite{GIST:13}, are required to solve \eqref{Prior.Art.o}.

\section{The Proposed Algorithm}\label{Sec:Proposed}

This work follows~\cite{Farahmand.robustRLS.12, Xiao.robustRLS.15}, but instead
of the classical RLS in \eqref{Prior.Art.F}, the recently introduced
hierarchical-optimization recursive least squares (HO-RLS)~\cite{SFMHSDM} is
used. In the current context, HO-RLS solves the following HO task: Given the
convex function $g: \Real^L \to \Real \cup \{+\infty\}$, which is generally
non-differentiable, find
\begin{align}
  \Argmin\nolimits_{\bm{F}}\ g(\bm{F})\
  \text{s.to}\ \bm{F}\in \Argmin_{\bm{F}'} \Expect \left\{ \tfrac{1}{2}
                 \norm{\vect{y}_n - \vect{o}_n -
                 \bm{F}' \vect{x}_n}_{\vect{R}_{vv}^{-1}}^2 
                 \right\} \,. \label{Goal}  
\end{align}
Unlike the RLS in \eqref{Prior.Art.F}, $g(\cdot)$ is able to quantify any
a-priori knowledge (side information) about the system. For example, if
$\vect{F}_*$ is known to be sparse, $g(\cdot) \coloneqq \norm{\cdot}_1$ can be
used to promote sparse solutions in \eqref{Goal}. To approximate the expectation
in \eqref{Goal}, the following sample-average loss is adopted:
$\forall n\geq n_0$,
$l_n(\vect{F}) \coloneqq [1/(2\Gamma_n)] \sum_{\nu=n_0}^n \gamma^{n-\nu}
\norm{\vect{y}_{\nu} - \hat{\vect{o}}_{\nu} - \vect{F}
  \vect{x}_{\nu}}_{\vect{R}_{vv}^{-1}}^2$, with a ``forgetting factor''
$\gamma\in (0,1]$ to mimic the classical exponentially-weighted RLS
scheme~\cite[\S30.6]{SayedBook}, and
$\Gamma_n \coloneqq \sum_{\nu=n_0}^{n} \gamma^{n-\nu}$. The outlier vector
$\vect{o}_n$ is replaced by its estimate $\hat{\vect{o}}_n$ in $l_n(\cdot)$.

Being an offspring of the stochastic Fej{\'e}r-monotone hybrid steepest-descent
method~\cite{SFMHSDM}, HO-RLS is based on the gradient of $l_n(\cdot)$. To this
end, given two stochastic processes $(\vect{a}_n)_n$ and $(\vect{b}_n)_n$,
define $\forall n\geq n_0$,
$\vect{R}_{ab,n} \coloneqq (1/\Gamma_n) \sum_{\nu=n_0}^n
\gamma^{n-\nu} \vect{a}_{\nu} \vect{b}_{\nu}^{\intercal}$, so that for the
processes $(\vect{x}_n)_n$, $(\vect{y}_n)_n$, and $(\hat{\vect{o}}_n)_n$ under
study, define $\vect{R}_{xx,n}$, $\vect{R}_{yx,n}$, and
$\vect{R}_{\hat{o}x,n}$. Then, the gradient of $l_n$ can be expressed as
$\nabla l_n (\vect{F}) = \vect{R}_{vv}^{-1} \left( \vect{F} \vect{R}_{xx,n} -
  \vect{R}_{yx,n} + \vect{R}_{\hat{o}x,n}\right)$, $\forall\vect{F}$. Following
the arguments of \cite[Algorithm~1 and (5a)]{SFMHSDM}, the previous gradient
information is incorporated into HO-RLS via the mapping
$T_n(\vect{F}) \coloneqq \vect{F} - (1/\varpi_n) \nabla\l_n (\vect{F})$,
$\forall\vect{F}$, with
$\varpi_n \geq \norm{\vect{R}_{vv}^{-1}} \,\norm{\vect{R}_{xx, n}}$, to produce
\cref{algo:ORHORLS}.

Any off-the-shelf solver can be employed to solve the sub-task
in~\cref{Algo:Step:LASSO}. Several solvers are explored here:
\begin{enumerate*}[label=\textbf{\roman*})]
\item The alternating direction method of multipliers
  (ADMM)~\cite{glowinski.marrocco.75, gabay.mercier.76} and
\item the recently developed Fej\'{e}r-monotone hybrid steepest descent method
  (FMHSDM)~\cite{FMHSDM}, which is the deterministic precursor of
  \cite{SFMHSDM}, in the case of $\rho(\cdot) \coloneqq \norm{\cdot}_1$; as well
  as
\item the general iterative shrinkage and thresholding (GIST)
  algorithm~\cite{GIST:13} in the case where $\rho(\cdot)$ takes the form of any
  non-convex surrogate of the $\ell_0$-norm, such as the minimax concave penalty
  (MCP)~\cite{Xiao.robustRLS.15}.
\end{enumerate*}
As it will be seen in \cref{Sec:Numerical.Tests}, only a small number of
iterations of the previous solvers suffice to provide the estimates
$\hat{\vect{o}}_n$ in \cref{Algo:Step:LASSO}.

In \cref{Algo:Step:Prox.initial,Algo:Step:Prox}, the proximal operator is
defined as
$\Prox_{\lambda g}(\vect{F}) \coloneqq \Argmin_{\bm{F}} (1/2) \norm{\vect{F} -
  \bm{F}}_{\text{Fr}}^2 + \lambda g(\bm{F})$, $\forall \vect{F}$, where
$\norm{\cdot}_{\text{Fr}}$ denotes the Frobenius norm. Clearly, if
$g(\cdot) = 0$, then $\Prox_{\lambda g}(\vect{F}) = \vect{F}$. The computational
complexity of \cref{Algo:Step:Prox} depends on the loss $g(\cdot)$. For example,
if $g(\cdot) = \norm{\cdot}_1$, then $\Prox_{\lambda g}(\cdot)$ operates on each
entry of its matrix argument separately, and boils down to the classical
soft-thresholding
mapping~\cite[Example~4.9]{HB.PLC.book}. \cref{Algo:Step:q.p,Algo:Step:varpi}
realize the power method~\cite{Golub:96} to provide running estimates of the
spectral norm $\norm{\vect{R}_{xx, n}}$, and the user-defined
$\epsilon_{\varpi}>0$ helps to provide an overestimate $\varpi_n$ of that
spectral norm.

Unlike the classical RLS (Newton's method in general)~\cite[\S9.8]{SayedBook},
HO-RLS uses $\nabla l_n(\cdot)$ in a way that avoids any matrix inversion
(lemma). The system updates which take place within
\Cref{Algo:Step:q.p,Algo:Step:Prox} of \cref{algo:ORHORLS} show a small
computational footprint, with the main burden being the matrix multiplications
in \cref{Algo:Step:Update.Fhalf}. Multiplications in
\cref{Algo:Step:Update.Fhalf} amount only to
$\vect{R}_{vv}^{-1} (\hat{\vect{F}}_n \vect{R}_{xx,n} - \vect{R}_{yx,n} +
\vect{R}_{\hat{o}x,n})$, since
$\vect{R}_{vv}^{-1} (\hat{\vect{F}}_{n-1} \vect{R}_{xx,n-1} - \vect{R}_{yx,n-1}
+ \vect{R}_{\hat{o}x,n-1})$ is already available from the previous recursion.

\begin{algorithm}[t]
  \DontPrintSemicolon
  \SetKwInOut{input}{Data}
  \SetKwInOut{parameters}{User's input}
  \SetKwInOut{output}{Output}
  \SetKwBlock{initial}{Initialization}{}

  \input{$(\vect{x}_n, \vect{y}_n)_{n\geq 0}$, $\vect{R}_{vv}^{-1}$.} 

  \parameters{$n_0$, $\alpha\in (0.5,1]$, $\lambda \in \RealPP$,
    $\epsilon_{\varpi} \in\RealPP$.}

  \output{Sequence $(\hat{\vect{F}}_n)_{n\geq n_0}$.}

  \BlankLine
  \initial{%

    $( \hat{\vect{F}}_{n_0}, \{\hat{\vect{o}}_{\nu, n_0}\}_{\nu=1}^{n_0} ) \in
    \Argmin\limits_{ (\bm{F}, \{\bm{o}_{\nu}\}_{\nu=1}^{n_0}) }
    \sum\limits_{\nu=1}^{n_0} \!\! \left[ \tfrac{1}{2n_0} \norm{
        \vect{y}_{\nu} - \bm{F}\vect{x}_{\nu} -
        \bm{o}_{\nu}}_{\vect{R}_{vv}^{-1}}^2 + \lambda \rho(\bm{o}_{\nu})
    \right]$. \label{Algo:Step:LASSO.initial}

    $\hat{\vect{o}}_{n_0} \coloneqq  \hat{\vect{o}}_{n_0, n_0}$.

    $\varpi_{n_0} \coloneqq \norm{\vect{R}_{vv}^{-1}} \,\norm{\vect{R}_{xx,n_0}} +
    \epsilon_{\varpi}$.

    $\hat{\vect{F}}_{n_0+1/2} \coloneqq \hat{\vect{F}}_{n_0} - 
    \tfrac{\alpha}{\varpi_{n_0}} \vect{R}_{vv}^{-1} ( \hat{\vect{F}}_{n_0}
    \vect{R}_{x, n_0} - \vect{R}_{yx, n_0} + \vect{R}_{\hat{o}x, n_0})$.

    $\hat{\vect{F}}_{n_0+1} \coloneqq \Prox_{\lambda
      g}(\hat{\vect{F}}_{n_0+1/2})$. \label{Algo:Step:Prox.initial}

  }

  \For{$n = n_0+1$ \KwTo $+\infty$}{%

    $\hat{\vect{o}}_n \in \Argmin\limits_{\bm{o}} \tfrac{1}{2} \norm{
      \vect{y}_{n} - \hat{\vect{F}}_n \vect{x}_{n} -
      \bm{o}}_{\vect{R}_{vv}^{-1}}^2 + \lambda
    \rho(\bm{o})$. \label{Algo:Step:LASSO}

    $\vect{q}_n \coloneqq \vect{R}_{xx,n} \vect{p}_{n-1}$;
    $\vect{p}_n \coloneqq \vect{q}_n / \norm{\vect{q}_n}$. \label{Algo:Step:q.p}

    $\varpi_n \coloneqq \norm{\vect{R}_{vv}^{-1}}\cdot (\vect{p}_n^{\intercal}
    \vect{R}_{xx,n} \vect{p}_n) + \epsilon_{\varpi}$. \label{Algo:Step:varpi}

    $\hat{\vect{F}}_{n+1/2} \coloneqq \hat{\vect{F}}_n + \hat{\vect{F}}_{n-1/2}
    - \hat{\vect{F}}_{n-1} + \tfrac{\alpha}{\varpi_{n-1}} \vect{R}_{vv}^{-1}
    (\hat{\vect{F}}_{n-1} \vect{R}_{xx,n-1} - \vect{R}_{yx,n-1} +
    \vect{R}_{\hat{o}x,n-1}) - \tfrac{1}{\varpi_n} \vect{R}_{vv}^{-1}
    (\hat{\vect{F}}_n \vect{R}_{xx,n} - \vect{R}_{yx,n} +
    \vect{R}_{\hat{o}x,n})$. \label{Algo:Step:Update.Fhalf}

    $\hat{\vect{F}}_{n+1} \coloneqq \Prox_{\lambda
      g}(\hat{\vect{F}}_{n+1/2})$. \label{Algo:Step:Prox}

  }

  \caption{Outlier-robust (OR-)HO-RLS}\label{algo:ORHORLS}
\end{algorithm}

\section{Convergence Analysis}\label{Sec:Convergence}

The convergence analysis of \cref{algo:ORHORLS} is based on the following set of assumptions.

\begin{assumption}\mbox{}
  \begin{asslist}

  \item\label{Ass:Ergodicity}
    $\vect{R}_{xx,n}\limas_n \vect{R}_{xx} \coloneqq \Expect\{\vect{x}_{n'}
    \vect{x}_{n'}^{\intercal}\}$ and
    $\vect{R}_{yx, n}\limas_n \vect{R}_{yx} \coloneqq \Expect\{ \vect{y}_{n'}
    \vect{x}_{n'}^{\intercal}\}$, $\forall n'\geq 1$, where $\limas_n$ denotes
    a.s.\ convergence~\cite{Williams.book.91}.
    
  \item\label{Ass:Uncorrelatedness} Both $(\vect{x}_n)_n$ and $(\vect{v}_n)_n$
    are zero-mean processes, and $(\vect{x}_n)_n$ is independent of
    $(\vect{o}_n)_n$ and $(\vect{v}_n)_n$.

  \item\label{Ass:Bounded.Variance} There exists $C\in\RealPP$ s.t.\ $\Expect\{
    \norm{ \vect{R}_{yx,n} }^2 \} \leq C$, $\forall n$.
    
  \item\label{Ass:Cond.Expectation} With
    $\mathcal{F}_{n} \coloneqq \sigma(\{\hat{\vect{F}}_{\nu}\}_{\nu=n_0}^n)$
    denoting the filtration ($\sigma$-algebra) generated by
    $\{\hat{\vect{F}}_{\nu}\}_{\nu = n_0}^n$, and
    $\Expect_{\given\mathcal{F}_{n}}\{\cdot\}$ being the conditional expectation
    given $\mathcal{F}_{n}$~\cite{Williams.book.91}, assume that
    $\Expect_{\given\mathcal{F}_{n}}\{\vect{R}_{xx,\nu}\} = \vect{R}_{xx}$ and
    $\Expect_{\given\mathcal{F}_{n}}\{\vect{R}_{yx,\nu}\} = \vect{R}_{yx}$,
    $\forall \nu\in \{n_0, \ldots, n\}$.

  \item\label{Ass:g} Consider a sequence $(\vect{Z}_n, \vect{\Xi}_n)_n$ s.t.\
    $(\vect{Z}_n, \vect{\Xi}_n) \in \domain g \times \partial g(\vect{Z}_n)$,
    $\forall n$, where
    $\domain g \coloneqq \Set{\vect{F}\in \Real^{L\times P} \given g(\vect{F}) <
      +\infty}$, $\partial g(\cdot)$ stands for the subdifferential mapping
    of $g$, \ie,
    $\partial g(\vect{Z}) \coloneqq \Set{\bm{\Xi}\in \Real^{L\times P} \given
      g(\vect{Z}) + \trace[(\vect{F}-\vect{Z})^{\intercal} \bm{\Xi}] \leq
      g(\vect{F}), \forall\vect{F} \in \Real^{L\times P}}$, and $\trace(\cdot)$
    denotes the trace of a square matrix. If $(\vect{Z}_n)_n$ is bounded a.s.,
    then $(\bm{\Xi}_n)_n$ is also bounded a.s. Moreover, if
    $(\Expect\{ \norm{\vect{Z}_n}_{\text{Fr}}^2 \})_n$ is bounded, then
    $(\Expect\{ \norm{\bm{\Xi}_n}_{\text{Fr}}^2 \})_n$ is also bounded. \qedhere
  \end{asslist}
\end{assumption}

\cref{Ass:Ergodicity} is motivated by ergodicity arguments~\cite{Petersen}, and
conditions which suffice to guarantee such an assumption can be based on laws of
large numbers through statistical independency or mixing
conditions~\cite{Andrews.88}. Bounded-moment assumptions, such as
\cref{Ass:Bounded.Variance}, appear frequently in stochastic approximation \eg,
\cite[p.~126, (A2.1)]{Kushner.Yin}. \cref{Ass:Cond.Expectation} is a sufficient
condition for the more relaxed and technical one in
\cite[Assumption~6]{SFMHSDM}. Due to space limitations, the stronger
\cref{Ass:Cond.Expectation} is used here. Moreover, many candidate losses
for $g(\cdot)$ satisfy \cref{Ass:g}. Examples are:
\begin{enumerate*}[label=\textbf{\roman*})]
\item the zero loss;
\item the indicator function $\iota_C(\cdot)$, used to enforce closed convex
  constraints $C\subset\Real^{L\times P}$ onto the desired solutions, with
  definition $\iota_C(\vect{F}) = 0$, if $\vect{F}\in C$, while
  $\iota_C(\vect{F}) = +\infty$, if $\vect{F}\notin C$;
\item $\norm{\cdot}_1$; and
\item $\norm{\cdot}^2$.
\end{enumerate*}

The following theorem is a consequence of Corollary~1 in~\cite{SFMHSDM}. The
subsequent proof translates the arguments of \cite{SFMHSDM} into the current
context. 

\begin{thm}\label{main.thm}
  Consider a large integer $n_{\sharp}$ and set $\hat{\vect{o}}_n$ equal to zero
  in \cref{Algo:Step:LASSO} of \cref{algo:ORHORLS}, $\forall n\geq
  n_{\sharp}$. Moreover, set
  $\varpi_n := \varpi \geq \max \{ \norm{\vect{R}_{vv}^{-1}}
  \,\norm{\vect{R}_{xx}}, \norm{\vect{R}_{vv}^{-1}} \,\norm{\vect{R}_{xx,n}}\}$,
  $\forall n > n_{\sharp}$.
  \begin{thmlist}

  \item\label{Thm:g} Under
    \Cref{Ass:Ergodicity,Ass:Uncorrelatedness,Ass:Bounded.Variance,Ass:Cond.Expectation,Ass:g},
    the set of cluster points of the sequence $(\hat{\vect{F}}_n)_{n\geq n_0}$ is
    non-empty, and any of those cluster points solves \eqref{Goal} a.s.

    \item\label{Thm:g=0} In the case where $g(\cdot) = 0$, and if
      $\vect{R}_{xx}$ is positive definite, then under
      \Cref{Ass:Ergodicity,Ass:Uncorrelatedness,Ass:Bounded.Variance,Ass:Cond.Expectation},
      the sequence $(\hat{\vect{F}}_n)_{n\geq n_0}$ generated by
      \cref{algo:ORHORLS} converges a.s.\ to the unknown $\vect{F}_*$. \qedhere

  \end{thmlist}

\end{thm}

\begin{proof}
  \textbf{(i)} According to \cite[Thm.~1 and Cor.~1]{SFMHSDM}, only
  Assumptions~2, 6, 7(ii) and 8 of \cite{SFMHSDM} need to be verified to
  establish \cref{Thm:g}. Assumptions~3 and 4 of \cite{SFMHSDM} are trivially
  satisfied due to the construction of problem \eqref{Goal}. The ergodicity
  Assumption~2 of \cite{SFMHSDM} is satisfied via \cref{Ass:Ergodicity},
  Assumption~7(ii) of \cite{SFMHSDM} is guaranteed by
  \cref{Ass:Bounded.Variance}, and Assumption~8 by \cref{Ass:g}. The proof of
  \cite{SFMHSDM} adapts to the present context via the following mappings: By
  the definition of $T_n$ in \cref{Sec:Proposed}, let
  $Q_n(\vect{F}) \coloneqq T_n(\vect{F}) - (1/\varpi_n) \vect{R}_{vv}^{-1}
  \vect{R}_{yx,n}$, while
  $T(\vect{F}) \coloneqq Q(\vect{F}) + (1/\varpi)\vect{R}_{vv}^{-1}
  \vect{R}_{yx}$, with
  $Q(\vect{F}) \coloneqq \vect{F} - (1/\varpi) \vect{R}_{vv}^{-1} \vect{F}
  \vect{R}_{xx}$ and
  $\varpi \geq \norm{\vect{R}_{vv}^{-1}} \,\norm{\vect{R}_{xx}}$. The
  application now of the conditional expectation
  $\Expect_{\given\mathcal{F}_{n}}\{\cdot\}$ to the terms that appear between
  (19b) and (19c) of \cite{SFMHSDM}, and define $\vartheta_n$, yields
  $\Expect_{\given\mathcal{F}_{n}}\{ \vartheta_n \} = 0$, a.s., and thus
  Assumption~6 of \cite{SFMHSDM} is satisfied by setting $\psi = 0$,
  a.s. Therefore, \cite[Thm.~1 and Cor.~1]{SFMHSDM} establish \cref{Thm:g}.

  \textbf{(ii)} Multiplying both sides of \eqref{data.model} from the right side
  by $\vect{x}_n^{\intercal}$ and by applying $\Expect\{\cdot\}$, it can be
  verified via \cref{Ass:Uncorrelatedness} that
  $\vect{F}_* = \vect{R}_{xy} \vect{R}_{xx}^{-1}$. The loss in the constraint of
  \eqref{Goal} becomes
  $\trace\{ \vect{F}^{\intercal} \vect{R}_{vv}^{-1} \vect{F} \vect{R}_{xx} -
  2 \vect{F}^{\intercal} \vect{R}_{vv}^{-1} \vect{R}_{xy} +
  \vect{R}_{vv}^{-1} \vect{R}_{yy}\}$. Hence, the minimizer
  $\vect{F}_{\text{opt}}$ of this loss satisfies the normal equations
  $\vect{R}_{vv}^{-1} \vect{F}_{\text{opt}} \vect{R}_{xx} = \vect{R}_{vv}^{-1}
  \vect{R}_{xy} \Rightarrow \vect{F}_{\text{opt}} = \vect{R}_{xy}
  \vect{R}_{xx}^{-1} = \vect{F}_*$. Since the minimizer is unique, \cref{Thm:g}
  suggests that all cluster points of $(\hat{\vect{F}}_n)_n$ coincide with
  $\vect{F}_{\text{opt}} = \vect{F}_*$.
\end{proof}

Setting $\hat{\vect{o}}_n$ equal to zero, for all sufficiently large $n$, has
been also used in the convergence analysis of
\cite{Farahmand.robustRLS.12}. Moreover, the assumption on the positive
definiteness of $\vect{R}_{xx}$ in \cref{Thm:g=0} holds true for any regular
stochastic process, \ie, a process with non-zero
innovation~\cite[Prob.~2.2]{Porat.book}.

\section{Numerical Tests}\label{Sec:Numerical.Tests}

To validate the contributions of this paper in methods and theory, extensive
numerical tests on synthetically generated data were conducted versus the
classical RLS~\cite{SayedBook} and the state-of-the-art
\cite{Farahmand.robustRLS.12, Xiao.robustRLS.15}. Two variants of
\cite{Farahmand.robustRLS.12} were employed to tackle~\eqref{Prior.Art.o} with
$\rho(\cdot) \coloneqq \norm{\cdot}_1$:
\begin{enumerate*}[label=\textbf{\roman*})]
\item The OR-RLS(ADMM) that employs ADMM~\cite{glowinski.marrocco.75,
  gabay.mercier.76}, and
\item a coordinate-descent approach OR-RLS(CD-L1)~\cite{Farahmand.robustRLS.12}.
\end{enumerate*}
With regards to \cite{Xiao.robustRLS.15}, two variants were considered to
solve~\eqref{Prior.Art.o} where $\rho(\cdot)$ takes the form of MCP:
\begin{enumerate*}[label=\textbf{\roman*})]
\item The OR-RLS(MCP) that employs the GIST method~\cite{GIST:13}, and
\item its coordinate-descent flavor OR-RLS(CD-MCP)~\cite{Xiao.robustRLS.15}.
\end{enumerate*}
The proposed OR-HO-RLS appears in three flavors, namely OR-HO-RLS(ADMM),
OR-HO-RLS(GIST) and OR-HO-RLS(FMHSDM), depending on the solver of
\cref{Algo:Step:LASSO} in \cref{algo:ORHORLS}. The performance metric is the
normalized root mean squared error
$\text{NRMSE} \coloneqq \norm{ \vect{F}_n - \vect{F}_* }_{\text{Fr}} / \norm{
  \vect{F}_* }_{\text{Fr}}$, where $\vect{F}_n$ stands for the estimate of any
of the employed methods at time $n$, and $\vect{F}_*$ is the unknown system in
\eqref{data.model}. The software code was written in Julia
(ver.~1.0.3)~\cite{Bezanson:julia.17}. The Julia package JuMP~\cite{JuMP:17} was
utilized, along with the Gurobi solver~\cite{Gurobi}, to solve
\eqref{Offline.task} as a warm start for \cite{Farahmand.robustRLS.12,
  Xiao.robustRLS.15} ($\rho(\cdot) = \norm{\cdot}_1$, $n = 500$), and similarly
\cref{Algo:Step:LASSO.initial} of~\cref{algo:ORHORLS} for OR-HO-RLS
($\rho(\cdot) = \norm{\cdot}_1$, $n_0 = 500$). Warm-start iterations and
times are not included in any of the subsequent numbers and figures. An
$100$ independent tests were performed and averaged values are reported. The
parameters of all methods ($\lambda$ for OR-HO-RLS), including those of any
off-the-shelf solver, were carefully tuned s.t.\ optimal performance is achieved
in all scenarios. In all cases, $\alpha \coloneqq 0.5$ and
$\epsilon_{\varpi} \coloneqq \num{5e-2}$ in \cref{algo:ORHORLS}.

In all tests, $P = 20$ and $L = 10$.  In all scenarios, noise $(\vect{v}_n)_n$
is modeled as Gaussian, zero-mean and colored, generated by an autoregressive
(AR) model, where the AR (state) matrix is randomly generated s.t.\ its maximum
singular value is $0.95$. To generate sparse outliers, the entries of
$\vect{o}_n$ are modeled as (also across time) independent and identically
distributed (IID) Bernoulli RVs with parameter $p_o$. Following
\cite{Farahmand.robustRLS.12}, the nonzero entries of $\vect{o}_n$ are drawn
from a uniform distribution with zero mean and variance $1\text{e}4$. The
entries of $(\vect{x}_n)_n$ are modeled as (also across time) IID Gaussian RVs
with zero mean and unit variance.

\Cref{Fig:Dense.20dB,Fig:Dense.10dB} refer to the ``stationary'' case where
$\vect{F}_*$ stays fixed $\forall n$. Matrix $\vect{F}_*$ is considered to be
``dense,'' \ie, with no zero entries. The entries are IID Gaussian RVs with zero
mean and unit variance. \cref{Fig:Dense.20dB} considers the case of
($\text{SNR} = 20\text{dB}$, $p_o = 0.2$), while \cref{Fig:Dense.10dB} the case
of ($\text{SNR} = 10\text{dB}$, $p_o = 0.1$). In these scenarios,
$g(\cdot) \coloneqq 0$ in \cref{algo:ORHORLS}. Moreover, in all stationary
scenarios $\gamma = 1$. All robust techniques outperform the classical
RLS. OR-RLS(ADMM) and OR-RLS(MCP) perform identically to their
coordinate-descent counterparts, while all flavors of OR-HO-RLS outperform all
other outlier-robust schemes. This behavior is also observed in
\cref{Fig:Dense.10dB}, but OR-RLS(MCP) and OR-RLS(CD-MCP) seem to reach the
levels of OR-HO-RLS(GIST) and OR-HO-RLS(FMHSDM). It is also worth noticing here
that OR-HO-RLS(FMHSDM) outperforms OR-HO-RLS(ADMM) even though both FMHSDM and
ADMM solve exactly the same $\ell_1$-norm penalized LS task
in~\cref{Algo:Step:LASSO} of \cref{algo:ORHORLS}.

A sparse stationary $\vect{F}_*$ is examined in \cref{Fig:Sparse.System}. Values
of $1$ are placed randomly at $10\%$ of the entries of $\vect{F}_*$, while the
rest of the entries are zero. Here, $g(\cdot) \coloneqq \norm{\cdot}_1$ in
\cref{algo:ORHORLS}. All curves exhibit similar behavior to that in
\Cref{Fig:Dense.20dB,Fig:Dense.10dB}. In \cref{Fig:Nonstationary}, a sudden
system change is introduced at time $2,500$, by randomly re-initializing
$\vect{F}_*$, to examine how fast the employed algorithms adapt to the
change. Here, $g(\cdot) \coloneqq 0$ in \cref{algo:ORHORLS}. In this scenario,
$\gamma = 0.97$ for all methods. \cref{Fig:Nonstationary} shows that
OR-HO-RLS(FMSHDM) exhibits fast adaptation to the system change while
maintaining the lowest levels of NRMSE among all methods.

\cref{Tab:Run.time} lists the computation times (in secs) per ``iteration'' on
an Intel(R) Xeon(R) CPU E5-2650v4 which operates at $2.20$GHz with $256$GB
RAM. ``Iteration'' refers to a single pass through \eqref{ORRLS} for
\cite{Farahmand.robustRLS.12, Xiao.robustRLS.15}, and to the iterations within
\Cref{Algo:Step:LASSO,Algo:Step:Prox} of \cref{algo:ORHORLS} for OR-HO-RLS. Each
iteration includes also the $100$ recursions of the off-the-shelf solvers which
are employed to generate the estimates $\hat{\vect{o}}_n$ in \eqref{Prior.Art.o}
and in~\cref{Algo:Step:LASSO} of~\cref{algo:ORHORLS}. It can be seen that the
proposed OR-HO-RLS(ADMM) and OR-HO-RLS(FMHSDM) are the fastest solutions among
all methods.

\begin{figure}	
  \centering
  \begin{subfigure}[]{.45\linewidth}
    \centering
    \includegraphics[width=\linewidth]{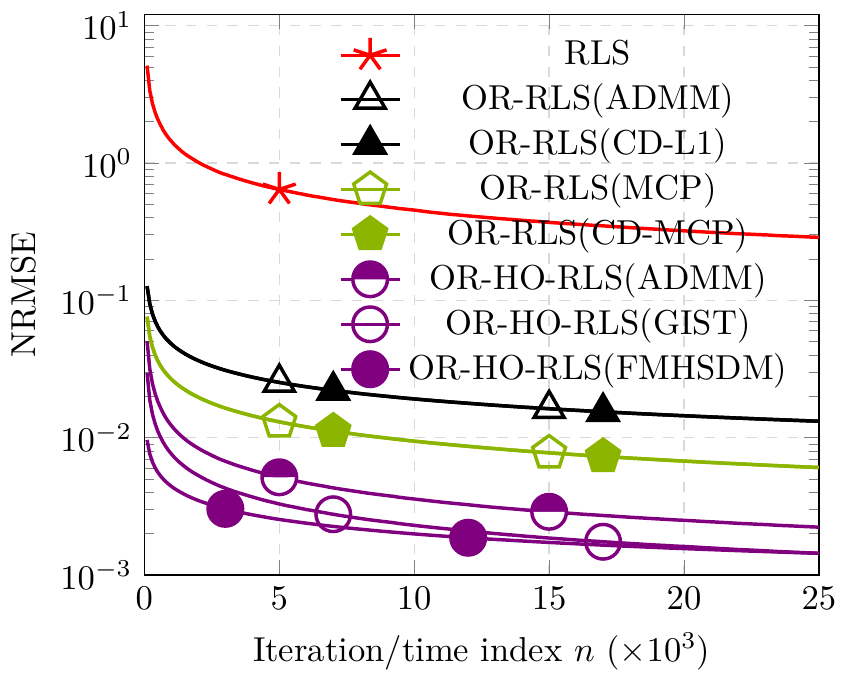}
    \subcaption{Dense and stationary $\vect{F}_*$; $\text{SNR} =
      20\text{dB}$; $p_o = 0.2$.}\label{Fig:Dense.20dB} 
  \end{subfigure}
  \begin{subfigure}[]{.45\linewidth}
    \centering
    \includegraphics[width=\linewidth]{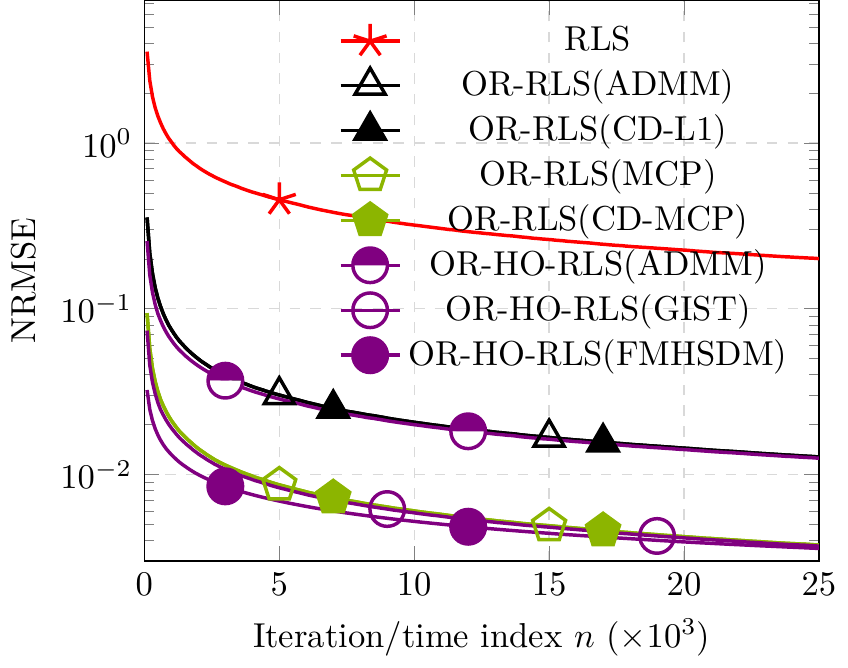}
    \subcaption{Dense and stationary $\vect{F}_*$; $\text{SNR} = 10\text{dB}$;
      $p_o = 0.1$.}\label{Fig:Dense.10dB}
  \end{subfigure}\\
  \begin{subfigure}[]{.45\linewidth}
    \centering
    \includegraphics[width=\linewidth]{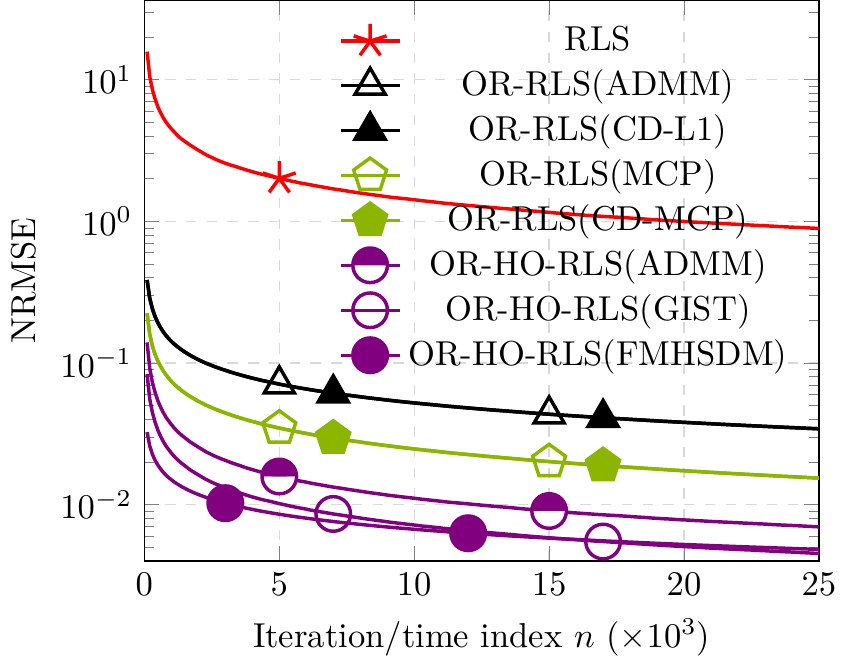}
    \subcaption{Sparse and stationary $\vect{F}_*$.
      $\text{SNR} = 20\text{dB}$; $p_o = 0.2$.}\label{Fig:Sparse.System}
  \end{subfigure}
  \begin{subfigure}[]{.45\linewidth}
    \includegraphics[width=\linewidth]{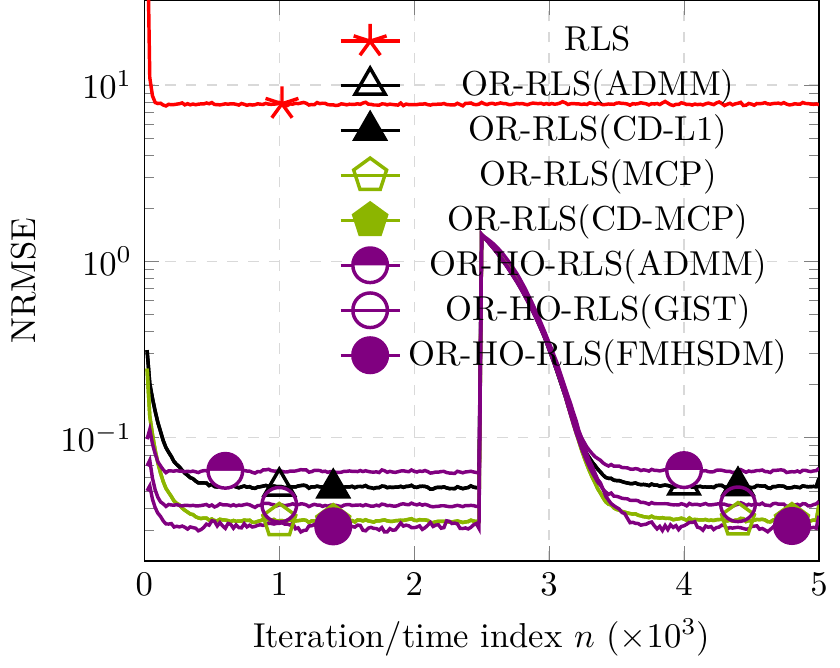}
    \subcaption{Dense and non-stationary $\vect{F}_*$; $\text{SNR} =
      20\text{dB}$; $p_o = 0.2$.}\label{Fig:Nonstationary}
  \end{subfigure}
  \caption{NRMSE values vs.\ iteration / time indices.}\label{Fig:NRMSE}
\end{figure}

\begin{table}
  \begin{center}
      \begin{tabular}{ ll }
        \hline
        Algorithm & Time (secs) / iteration \\
        \hline
        RLS & \num{0.00412} $\pm$ \num{7.85820e-7} \\
        OR-RLS(ADMM) & \num{0.00357} $\pm$ \num{3.07575e-7} \\
        OR-RLS(CD-L1) & \num{0.00747} $\pm$ \num{1.23128e-6} \\
        OR-RLS(MCP) & \num{0.01511} $\pm$ \num{5.78743e-6} \\
        OR-RLS(CD-MCP) & \num{0.00854} $\pm$ \num{1.68219e-6} \\
        \bfseries OR-HO-RLS(ADMM) & \bfseries \num{0.00111} $\pm$ \num{3.94734e-8}
        \\
        OR-HO-RLS(GIST) & \num{0.01240} $\pm$ \num{4.06283e-6} \\
        OR-HO-RLS(FMHSDM) & \num{0.00279} $\pm$ \num{2.52144e-7} \\
        \hline
      \end{tabular}
  \end{center}
  \caption{Average values $\pm$ standard deviations of the run times (in secs)
    per iteration for the scenario of
    \cref{Fig:Dense.20dB}.}\label{Tab:Run.time}
\end{table}



\end{document}